\documentclass[twoside,11pt]{article}

\usepackage{blindtext}

\usepackage{jmlr2e_review}

\usepackage[utf8]{inputenc}
\usepackage{graphicx}
\usepackage{xcolor}
\usepackage{amsfonts}
\usepackage{amsmath}
\usepackage{amssymb}
\usepackage{mathtools}
\usepackage{hyperref}
\usepackage{array}
\usepackage{multirow}

\newcommand{\R}{\mathbb{R}}

\newcommand{\N}{\mathcal{N}}

\newcommand{\hba}{\hat{\beta}_A}
\newcommand{\hbb}{\hat{\beta}_B}
\newcommand{\hbba}{\hat{\beta}_{BA}}
\newcommand{\hb}{\hat{\beta}}

\newcommand{\eps}{\epsilon}
\newcommand{\s}{\sigma}
\newcommand{\gam}{\gamma}
\newcommand{\PAT}{\mathcal{P}_{{A^\top}}}
\newcommand{\rvect}[1]{\begin{bmatrix} #1 \end{bmatrix}}
\newcommand*\cvect[3][]{
    \begin{bmatrix}\ifx\relax#1\relax\else#1\\\fi#2\\#3\end{bmatrix}
}

\usepackage{lastpage}
\jmlrheading{}{}{1-\pageref{LastPage}}{2/25}{X/XX}{21-0000}{Daniel Goldfarb and Paul Hand}

\ShortHeadings{}{Daniel Goldfarb and Paul Hand}
\firstpageno{1}

\begin{document}

\title{Analysis of Overparameterization in Continual Learning under a Linear Model}

\author{\name Daniel Goldfarb \email goldfarb.d@northeastern.edu \\
       \addr Khoury College of Computer Sciences\\
       Northeastern University\\
       Boston, MA 02115, USA
       \AND
       \name Paul Hand \email p.hand@northeastern.edu \\
       \addr Khoury College of Computer Sciences\\
       Northeastern University\\
       Boston, MA 02115, USA}

\editor{}

\maketitle

\begin{abstract}

Autonomous machine learning systems that learn many tasks in sequence are prone to the catastrophic forgetting problem. Mathematical theory is needed in order to understand the extent of forgetting during continual learning. As a foundational step towards this goal, we study continual learning and catastrophic forgetting from a theoretical perspective in the simple setting of gradient descent with no explicit algorithmic mechanism to prevent forgetting. In this setting, we analytically demonstrate that overparameterization alone can mitigate forgetting in the context of a linear regression model. We consider a two-task setting motivated by permutation tasks, and show that as the overparameterization ratio becomes sufficiently high, a model trained on both tasks in sequence results in a low-risk estimator for the first task. As part of this work, we establish a non-asymptotic bound of the risk of a single linear regression task, which may be of independent interest to the field of double descent theory.

\end{abstract}


\section{Introduction}

The field of lifelong learning aims to develop machine learning systems that can continually learn many tasks by transferring and retaining knowledge from previously learned tasks without the need for retraining (\cite{chen2022lifelong}). This is an important problem because it can be a waste of computation to retrain neural networks from scratch. One challenge of continual learning is the effect of catastrophic forgetting, where models abruptly lose performance on previously learned tasks when training on new ones (\cite{french1999catastrophic}). The amount of expected catastrophic forgetting is affected by multiple aspects of the tasks, like the neural network architecture and task similarity or difficulty. Hence the field of continual learning aims to develop dedicated approaches to mitigate catastrophic forgetting. These include regularization-based methods, architectural modifications, and memory-replay methods (\cite{li2017learning, kirkpatrick2017overcoming, zenke2017continual, shin2017continual}).

In order to build an understanding of state-of-the-art continual learning methods, it is a prerequisite to understand simpler methods. One such simple method is naively modifying the overparameterization level. Modern neural networks typically operate in the extremely overparameterized regime, where the number of learned parameters is much larger than the number of samples in the training set. In practice overparameterized models perform better than their underparameterized counterparts (\cite{nakkiran2021deep, zhang2021understanding}). This phenomenon, coined double descent, sparked active research to explain this behavior (\cite{adlam2020understanding, blanc2020implicit}). While most double descent works contribute to the single-task setting, one continual learning related work (\cite{mirzadeh2022wide}) ran experiments on popular rotated image benchmarks to show that overparameterized neural networks may be less prone to forgetting than their moderately parameterized counterparts. \cite{goldfarb2023analysis} replicated these results for random data permutation tasks.

Mathematical theory is necessary in order to understand when continual learning algorithms are needed and when to expect performance degradation in practice. This theory is challenging because the theoretical analysis of trained neural networks is still a nascent field, and additional architectural and algorithmic modifications would increase the challenge of that analysis. As a foundation for understanding the algorithmic developments of continual learning, we aim to understand the simplest setting of gradient descent with linear models modified only by their overparameterization level (\cite{evron2023continual, lin2023theory}). This setting has been previously studied by theoretical papers for double descent in the single-task setting (\cite{muthukumar2020harmless, belkin2020two, bartlett2020benign, hastie2022surprises}). As summarized in \cite{dar2021farewell}, these works concurrently found that a linear model can enjoy optimal performance in the overparameterized regime when the effective dimensionality of the data is lower than the number of training data samples, and the true dimensionality of the data is higher than the number of training data samples. We aim to leverage existing single-task work as a foundation to quantify the effect of overparameterization in the multi-task setting. The model presented in \cite{hastie2022surprises} has a strong connection to neural networks, which we fully describe in Section \ref{sec:model}, and we take inspiration from this model in our work.

In this paper, we study a two-task continual learning setting with a latent space data model inspired by \cite{hastie2022surprises}. The tasks will be related by a random orthogonal transformation. This two-task setup is a mathematical idealization of the standard continual learning benchmark where tasks are related by a random permutation (\cite{kirkpatrick2017overcoming}). We demonstrate that, under this linear model, if overparameterization is sufficiently high and there are sufficiently many training examples relative to the model's latent dimensionality, then this two-task setup has low risk on task A both before and after training on task B. 

We now present an informal version of our main conclusion. Consider tasks A and B each with $n$ training examples and $p$ observed features but only $d$ latent features. We are interested in the overparameterized regime where $d < n < p$. The data of each task is related by a random $p \times p$ orthogonal matrix. Given task A with data $X_A \in \R^{n \times p}$, task B is given by $X_B = X_AO^\top$. Both tasks have shared labels $y \in \R^n$ as is standard with permutation tasks. Our training procedure will fit task A by minimizing the square loss from initialization 0, and then subsequently training on task B with no explicit mechanism to prevent forgetting. Let $\hba$ be the parameters of the model that is trained on task A and $\hbba$ be the parameters of the model that is subsequently trained on task B. Our main result establishes non-asymptotic bounds with high probability on the risk of task A of several estimators. We informally present it as follows:

\begin{align}
    R(\hba) &\lesssim \sqrt{\frac{d}{n}} + \frac{n}{p}, \label{eq:intro-single} \\
    R(\hbba) &\lesssim \sqrt{\frac{d}{n}} + \sqrt{\frac{n}{p}}, \label{eq:intro-terminal} \\
    \frac{R(\hbba) - R(\hba)}{R(0) - R(\hba)} &\lesssim \frac{\sqrt\frac{n}{p}}{1 - C\left(\sqrt{\frac{d}{n}} + \frac{n}{p}\right)}, \label{eq:intro-ratio}
\end{align}

\noindent
where $R(\cdot)$ denotes the risk on task A and $C$ is a constant. Equation (\ref{eq:intro-single}) quantifies the performance of an estimator trained on a single task. Equation (\ref{eq:intro-terminal}) quantifies the performance on task A of the estimator that trains on task A then task B in sequence. Equation (\ref{eq:intro-ratio}) bounds the ratio between the forgetting due to training on task B relative to the amount of initial learning of task A. Observe that all three bounds are small when the number of latent features is much smaller than the number of training examples, and the number of training examples is much smaller than the number of features. Hence task A is well-learned initially, task A remains well-learned after training on task B, and the forgetting due to training on task B is insignificant in comparison to the amount of learning on task A when initialized at 0. These conditions are the same as the ones summarized by \cite{dar2021farewell} to describe when double descent occurs for the single-task case. The formal version of this result is presented in Section \ref{sec:theory} and its proof is provided in Appendix \ref{appendix-theory}. This work is an extension of a conference paper with the same authors as the present work (\cite{goldfarb2023analysis}) that quantified only $R(\hbba) - R(\hba)$. That work left open the questions of whether task A was well-learned in the first place and whether it is still well-learned after training on task B, which we resolve in the present work.

While the present paper's conclusions are framed around the two-task continual learning setting, its primary technical novelty is a non-asymptotic bound on the single-task risk in the linear regression model mentioned above. Existing work derived a risk expression in the asymptotic regime (\cite{hastie2022surprises}). In order to obtain results for problems with finite size, we establish a non-asymptotic version of this result which may be of independent interest. The proof uses random matrix estimates from \cite{vershynin2010introduction}. A more thorough comparison of our work and \cite{hastie2022surprises} is provided in Section \ref{sec:theory}.

The rest of the paper is organized as follows. In Section \ref{sec:model}, we formally present the statistical model of the two-task continual learning setup.  In Section \ref{sec:theory}, we present our main result which is a theorem that formalizes (\ref{eq:intro-single})-(\ref{eq:intro-ratio}). In Section \ref{sec:conclusion}, we discuss the limitations and potential future directions of our work.


\section{Latent Space Regression Model} \label{sec:model}

In this section, we first present the latent space model from \cite{hastie2022surprises} and share the analogy between it and neural networks on vision tasks. Then we present an extension of the latent space model to the setting of two orthogonal transformation tasks as defined in \cite{goldfarb2023analysis}.

\subsection{Single-Task Latent Space Model} \label{one-task-model}

Consider a linear regression model with $d$-dimensional standard Gaussian latent features. Let the responses be generated by a noiseless inner product with some ground truth parameters $\theta \in \R^d$:

\begin{align}
    z &\sim \N(0, I_d), \label{data-start} \\
    y &= z^\top\theta.
\end{align}

\noindent
Instead of performing regression on the latent features, suppose we only have access to noisy $p$-dimensional random projections, $x$, of the latent features:

\begin{align}
    x &= Wz + u, \label{data-end}
\end{align}

\noindent
where $W \in \R^{p \times d}$ and $u \sim \N(0, I_p)$. Here the number of noisy features, $p$, controls the overparameterization level of the problem. While $W$ could be defined to have Gaussian entries, for mathematical convenience we consider the idealized case where $W^\top W 
= p\gam I_p$. $\gam$ acts as a signal-to-noise term and columns of $W$ are pairwise orthogonal and have equal length.

The regression problem provides $n$ training samples and corresponding responses in the form of data $(X, y) \in \R^{n \times p} \times \R^n$, where each of the $n$ rows of $X$ and entries of $y$ are sampled independently by (\ref{data-start}) - (\ref{data-end}).

Consider a model that makes predictions on the observed features linearly:

\begin{align}
    f_{\hb} : x \mapsto x^\top\hb,
\end{align}

\noindent
where $\hb \in \R^p$ are the parameters of the model. We train the parameters using gradient descent with a square loss. Due to the overparameterized nature of the problem ($d<n<p$), the result of gradient descent will depend on initialization and converge to a solution that fits to the training data. We consider solving $\arg\min_{\hb} \frac{1}{2}\|X\hb-y\|^2$ using gradient descent from initialization $\beta_0$. There is an analytical solution to this given by

\begin{align}
    \arg\min_{\hat{\beta}} \frac{1}{2}\|\hat{\beta} - \beta_0\|^2 \text{ s.t. } y = X\hb. \label{opt-prob}
\end{align}

\noindent
The optimization problem in (\ref{opt-prob}) has the following closed form solution when $X$ is full-rank (which occurs with probability $1$ under the above model):

\begin{align}
    \hat{\beta} &=\beta_0 + X^\top(XX^\top)^{-1}y-X^\top(XX^\top)^{-1}X\beta_0 \\
    &=\beta_0 + X^\top(XX^\top)^{-1}y-\mathcal{P}_{X^\top}\beta_0,
\end{align}

\noindent
where $\mathcal{P}_{X^\top}$ is the orthogonal projector onto the range of $X^\top$. We study the performance of an estimator by its statistical risk, which is given by

\begin{align}
    R(\hb) = \s^2 + (\hb-\beta)^\top\Sigma(\hb-\beta), \label{risk}
\end{align}

\noindent
where
\begin{align}
    \Sigma &= WW^\top + I_p, \\
    \beta &= (I+WW^\top)^{-1}W\theta, \\
    \sigma^2 &= \theta^\top(W^\top W + I_d)^{-1}\theta.
\end{align}

\noindent
The derivation of (\ref{risk}) follows from showing that the data model of (\ref{data-start}) - (\ref{data-end}) is equivalent to a modified model where the noise occurs on the responses instead of the observed features (\cite{hastie2022surprises}). This modified model is more amenable to analysis and its equivalence is detailed in Appendix E of \cite{goldfarb2023analysis}.

\begin{table} \caption{Analogy between neural networks and the linear model from Section \ref{one-task-model}.} 

\label{analogy-table}
\begin{center}
\begin{tabular}{ | m{3.1cm} | m{5.2cm}| m{5.2cm} | } 
    \hline
     & \textbf{Linear Model in Section \ref{one-task-model}} & \textbf{Neural Networks for Vision}  \\
  \hline
  Latent Features & Labels are generated from low-dimensional latent features $z$. & Many classes of natural images have an underlying low dimensionality. \\ 
  \hline
  Observed Features & Model has access to high-dimensional random features, $x$. & Model has access to high-dimensional pixel representation. \\ 
  \hline
  Noise within Observed Features & Observed features contain additive Gaussian noise controlled by $\gam$. & Observed features may contain noise, for example due to sensor approximation. \\ 
  \hline
\end{tabular}
\end{center}
\end{table}

The data distribution of the above model mirrors the supervised learning setting of modern neural networks. We flesh out this analogy in Table \ref{analogy-table}.

\subsection{Two-Task Latent Space Model} \label{two-task-model}

We now extend the data model of Section \ref{one-task-model} to study two tasks which are related by a random orthogonal transformation. This provides tasks that are equally difficult (for a fully-connected model) and are similar to the random data permutation benchmark of continual learning (\cite{de2021continual}).

Consider two tasks, $A$ and $B$, with $n$ examples each. Task $A$ has data $(X_A, y) \in \R^{n \times p} \times \R^n$ where each of the $n$ rows of $X_A$ and entries of $y$ are sampled independently by (\ref{data-start}) - (\ref{data-end}) for a fixed $\theta$. Define $O$ to be sampled uniformly from the set of $p \times p$ orthogonal matrices. Task B then has data $(X_B, y)$ where $X_B = X_A O^\top$. Observe that $y$ is the same for tasks A and B as is standard for random permutation tasks.

Recall from Section \ref{one-task-model} that each estimator is given by a solution to a least squares problem from a variety of initializations. Let $\hba$ be the solution to task A when initialized at $0$, $\hbb$ be the solution to task B when initialized at $0$, and $\hbba$ be the solution to task B when initialized at $\hba$. Using the closed form solutions from Section \ref{one-task-model}, we get the following formulas for these parameters:

\begin{align}
    \hba &= X_A^\top(X_AX_A^\top)^{-1}y, \label{eq-beta-A} \\
    \hbb &= X_B^\top(X_BX_B^\top)^{-1}y, \\
    \hbba &= \hba + \hbb - \mathcal{P}_{X_B^\top}\hba. \label{eq-beta-BA}
\end{align}

\noindent
Under the two-task model, we are interested in the risk of these estimators on task A as given by (\ref{risk}).

This two-task formulation was first used in  \cite{goldfarb2023analysis} which provided an analytical illustration of the previously observed empirical behavior that overparameterized neural networks can be less prone to catastrophic forgetting (\cite{mirzadeh2022wide}). This empirical observation was made on the popular image rotation/permutation continual learning benchmarks. Under this setting, the first task of the problem is the original image classification task, and subsequent tasks are random rotations/permutations of the original task. This provides tasks that are nearly orthogonal, thus allowing for subsequent training with minimal forgetting (\cite{farajtabar2020orthogonal}). \cite{goldfarb2023analysis} aimed to develop a linear regression data model that mirrored relevant permuted image benchmarks for their analysis. We use the same model in this work for its mathematical tractability.

\section{Main Result} \label{sec:theory}

Under the two-task linear model introduced in Section \ref{two-task-model}, our main result quantifies the risk on task A of estimators (\ref{eq-beta-A})-(\ref{eq-beta-BA}). This provides parameter regimes that guarantee that:

\begin{itemize}
    \item Task A is initially well-learned\footnote{We will say that a task is well-learned when it has risk significantly smaller than that of the null estimator.}.
    \item After subsequent training on task B, the risk on task A remains well-learned.
    \item The amount of forgetting is small compared to the amount of initial learning.
\end{itemize}

\noindent
For mathematical convenience, suppose $W$ satisfies the following assumption:

\begin{assumption} \label{body-assumption}
    All columns of $W$ are pair-wise orthogonal and have equal length. That is, $W^\top W = p\gamma I_d$.
\end{assumption}

\noindent
This assumption is an idealization of Gaussian $W$ which was also made in \cite{hastie2022surprises}. 

We will compare the performance of our learned estimators to the 0 predictor, which we call the (unlearned) null risk baseline. To demonstrate that task A is well-learned, it must outperform this null risk. By Lemma F.3 in \cite{goldfarb2023analysis}, null risk has an exact value of $R(0) = \|\theta\|^2$. Our main result is the following:

\begin{theorem}\label{thm:main-body} Fix $\theta \in \R^d$. Let tasks $A,B$ be given by Section \ref{two-task-model}. Let $W \in \R^{p \times d}$ satisfy Assumption \ref{body-assumption} and $n \ge d, p \ge 20n$, $\gam \ge \frac{1}{\sqrt{nd}}$. Then there exists constant $c>0$ such that with probability at least $1-20e^{-cd}$, the following holds:

\begin{align}
    R(\hba) &\le \left(72\sqrt{\frac{d}{n}} + \frac{18n}{p}\right)\|\theta\|^2, \label{eq:single} \\
    R(\hbba) &\le \left(72\sqrt{\frac{d}{n}} + 96\sqrt{\frac{n}{p}}\right)\|\theta\|^2, \label{eq:terminal} \\
    \frac{R(\hbba) - R(\hba)}{R(0) - R(\hba)} &\le \frac{78\sqrt\frac{n}{p}}{1 - 72\sqrt{\frac{d}{n}} - \frac{18 n}{p}}. \label{eq:ratio}
\end{align}

\end{theorem}

\noindent
Equation (\ref{eq:single}) quantifies the risk on task A of the estimator trained exclusively on task A. Thus observe that when $n$ scales linearly with $d$ and $p$ scales linearly with $n$, the risk is small and task A is well-learned. Equation (\ref{eq:terminal}) quantifies the performance on task A of an estimator that trains on task A then task B in sequence. Under the same linear scalings as before, observe that task A stays well-learned after training on task B. Finally, equation (\ref{eq:ratio}) asserts that the fraction of performance loss as a result of forgetting compared to the performance gain due to original training is small under this parameter regime. 
Note that as $p \rightarrow \infty$, the bounds we provide on $R(\hba)$ and $R(\hbba)$ are monotonically decreasing but do not go to 0. Indeed performance is limited by the ratio of latent dimensionality and number of training samples. The bound on the risk ratio in (\ref{eq:ratio}) does go to 0, thus forgetting is fully ameliorated for this model in this setting. The conclusions about risk hold under a wide range of signal-to-noise ratios, even as low as $\gam=\frac{1}{\sqrt{nd}}$. Observe that the failure probability is exponentially small in $d$. This is due to the fact that the proof is based on non-asymptotic random matrix results. Each of (\ref{eq:single})-(\ref{eq:ratio}) are proven in Theorems \ref{single-risk}, \ref{term-A}, and \ref{scale-forgetting} respectively.

The bound in (\ref{eq:single}) may be of independent to the field of double descent theory. It is a non-asymptotic result similar to the latent space result in \cite{hastie2022surprises}. Their result established that if $d, n, p \rightarrow \infty$ with fixed ratios $\frac{d}{p}, \frac{n}{p}$ then an explicit expression for risk can be derived that exhibits double descent. This result is asymptotic in nature. In order to obtain a risk estimate for problems with finite size, we establish a non-asymptotic version of this result. Our result relies on non-asymptotic singular value estimates, such as those in \cite{vershynin2010introduction}. As (\ref{eq:single}) is monotonically decreasing in $p$, it demonstrates a double descent effect and hence this technical novelty may be of independent interest to the overparameterized machine learning theory community.

\section{Discussion and Conclusion} \label{sec:conclusion}

In this work we present a two-task linear model where data is generated from a latent space model and tasks are related by a random orthogonal transformation. We show theoretically that when the model is suitably overparameterized, forgetting is small and allows for high performance on the first task both before and after training on the second task.

\begin{figure}
\begin{center}
\includegraphics[scale=0.23]{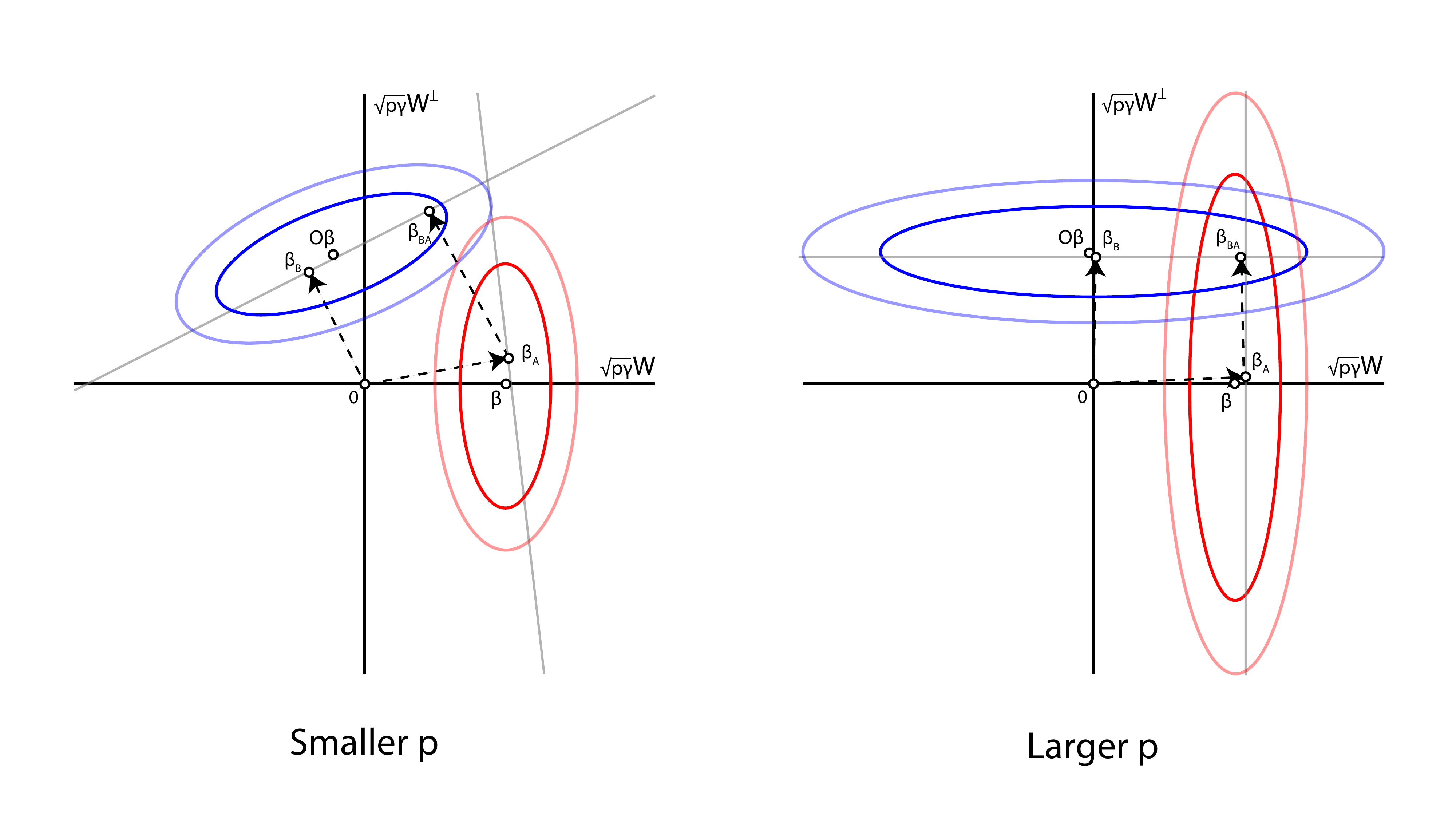}
\caption{Vertical and horizontal axes (solid black lines) depict the span of $W$ and its orthogonal complement (scaled for ease of visualization). The true parameters of task A are given by $\beta$ and the true parameters of task B are given by $O\beta$. Notice that $\beta$ lies directly on $W$ and $O\beta$ lies further from $W^\perp$ when $p$ is smaller but closer to $W^\perp$ when $p$ is larger. The lighter solid lines (gray) depict the set of solutions to $X_A\beta=y$ and $X_B\beta=y$. The predictors $\hba, \hbb, \hbba$ are given by orthogonally projecting from initialization onto these lines, but $\beta$ and $O\beta$ do not lie exactly on the gray lines due to the noise in the observed features. The left (red) and right (blue) ellipses depict level sets of constant risk for tasks A and B respectively. Notice that the aspect ratio of these ellipses is more extreme for larger $p$ and under this regime $\hbba$ has similar risk on task A as $\hba$.}
\label{fig:illustration}
\end{center}
\end{figure}

Overparameterization provides the benefit that catastrophic forgetting is minimal. This effect can be seen geometrically in Figure \ref{fig:illustration}. As seen in this figure, level sets of constant risk become orthogonal to each other and get extremely high aspect ratios in the overparameterized regime. Thus the learned predictor is able to train on each task in sequence and retain high performance on both tasks at the end of training.


We hope that this work may inspire several avenues of future research. Our work was inspired by random data permutation tasks for their mathematical tractability but an extension to other relevant continual learning benchmarks like split tasks would be a beneficial step towards a better understanding of forgetting regimes. Additionally, the field would benefit from the direct analysis of existing continual learning methods with the goal of better understanding why they work well in practice. Lastly, the analysis of more realistic non-linear or even neural network models would provide a theoretical understanding that is more closely tied to the techniques currently used in practice. Eventually we hope that such theoretical analysis can help in the development of future continual learning approaches.


\acks{DG and PH acknowledge support from NSF Awards DMS-2053448, DMS-1848087. PH acknowledges further support from NSF Award DMS-2022205.}


\newpage

\appendix

\section{} \label{appendix-theory}

\begin{assumption} \label{assumption}
    All columns of $W$ are pair-wise orthogonal and of equal length. Namely, $W^\top W = p\gamma I_d$.
\end{assumption} 

\noindent
A consequence of Assumption \ref{assumption} is that $WW^\top = p\gam \mathcal{P}_W$. This formally proven in Lemma F.2 of \cite{goldfarb2023analysis}.

\begin{lemma} \label{beta-norm} Suppose $W \in \R^{p \times d}$ follows Assumption \ref{assumption}. Define $\beta = W(W^\top W + I_d)^{-1}\theta$ for some $\theta \in \R^d$. Then

\begin{align}
    \|\beta\|^2 = \frac{p\gam}{(p\gam+1)^2}\|\theta\|^2
\end{align}

\end{lemma}

\begin{proof} By definition,

\begin{align}
    \|\beta\|^2 = \theta^\top(W^\top W + I_d)^{-1}W^\top W (W^\top W + I_d)^{-1} \theta
\end{align}

\noindent
Using Assumption \ref{assumption},

\begin{align}
    \|\beta\|^2 &= \theta^\top(p\gamma I_d + I_d)^{-1}p\gamma I_d (p\gamma I_d + I_d)^{-1} \theta \\
    &= \frac{p\gam}{(p\gam+1)^2}\|\theta\|^2
\end{align}
    
\end{proof}

\begin{lemma} \label{s-approx}
    Assume $W \in \R^{p \times d}$ satisfies Assumption \ref{assumption}. Then
    
    \begin{align}
        WW^\top = p\gamma\mathcal{P}_W
    \end{align}

    \noindent
    where $\mathcal{P}_W$ is the orthogonal projection onto the range of $W$.
\end{lemma}

\begin{proof} Lemma F.2 in \cite{goldfarb2023analysis}.

\end{proof}

\begin{lemma} \label{null-risk}
    Let $\Sigma = WW^\top+I_p$ where $W \in \R^{p \times d}$ satisfies Assumption \ref{assumption}. For some $\theta \in \R^d$, let $\beta = W(W^\top W + I_d)^{-1}\theta$. Then
    
    \begin{align}
        R(f_{0}) = \|\theta\|^2
    \end{align}
\end{lemma}

\begin{proof} Lemma F.3 in \cite{goldfarb2023analysis}.

\end{proof}

\begin{lemma} \label{term3}
    Let $W=\cvect{\sqrt{p\gam}I_d}{0} \in \R^{p \times d}$ and define data matrix $A \in \R^{n \times p}$ whose rows are independent random vectors in $\R^p$ with covariance $\Sigma = WW^\top + I_p$ and $d \le n$, $2n \le p$. Assume $\gam \ge \frac{1}{\sqrt{nd}}$. Then with probability at least $1-8e^{-cd}$ for universal constant $c>0$,

    \begin{align}
        \|\mathcal{P}_{{A^\top}^\perp}e_1\|^2 \le 18\sqrt{\frac{d}{n}}
    \end{align}

    \noindent
    where $\mathcal{P}_{{A^\top}^\perp}$ is the orthogonal projection on the orthogonal complement of the range of $A^\top$ and $e_1$ is the $p$-dimensional first standard basis vector.
\end{lemma}

\begin{proof} We have that
    \begin{align}
    \|\mathcal{P}_{{A^\top}^\perp}e_1\|^2 &= \|e_1\|^2 - \|\mathcal{P}_{{A^\top}}e_1\|^2 \\
    &= 1 - \|\mathcal{P}_{{A^\top}}e_1\|^2 \label{one-minus}
\end{align}

\noindent
So it suffices to study $\|\mathcal{P}_{{A^\top}}e_1\|^2 = e_1^\top \mathcal{P}_{{A^\top}} e_1$:

\begin{align}
    e_1^\top \PAT e_1 = e_1^\top A^\top(AA^\top)^{-1}A e_1
\end{align}

\noindent
Let $A = \rvect{A_1 & A_2}$ where $A_1 \in \R^{n \times d}, A_2 \in \R^{n \times p-d}$, then

\begin{align}
    e_1^\top \PAT e_1 &= e_1^\top A_1^\top (A_1A_1^\top + A_2A_2^\top)^{-1} A_1 e_1 \\
    &\ge e_1^\top A_1^\top (A_1A_1^\top + s_{max}(A_2A_2^\top)I_n)^{-1} A_1 e_1
\end{align}

\noindent
Where the last inequality follows from $A_1A_1^\top + A_2A_2^\top \preceq A_1A_1^\top + s_{max}(A_2A_2^\top)I_n$. By Theorem 5.39 of \cite{vershynin2010introduction}, $s_{max}(A_2) \le \sqrt{p-d}+2\sqrt{n}$ on an event of a given probability. So $s_{max}(A_2A_2^\top) \le (\sqrt{p-d}+2\sqrt{n})^2$. Also under the assumptions that $p \ge 2n$ and $n \ge d$, we have that $n \le p-d$, and further $\sqrt{p-d}+2\sqrt{n} \le 3\sqrt{p-d}$. This gives

\begin{align}
    e_1^\top \PAT e_1 &\ge  e_1^\top A_1^\top (A_1A_1^\top + 9(p-d)I_n)^{-1} A_1 e_1
\end{align}

\noindent
Let $A_1 = USV^\top$ be the SVD of $A_1$ where $U \in \R^{n \times n}, S \in \R^{n \times d}, V \in \R^{d \times d}$. Then

\begin{align}
    e_1^\top \PAT e_1 &\ge e_1^\top VS^\top U^\top (U(SS^\top+9(p-d)I_n)^{-1}U^\top) USV^\top e_1 \\
    &= e_1^\top VS^
    \top(SS^\top+9(p-d)I_n)^{-1}SV^\top e_1 \\
    &= e_1^\top V\text{diag}\left(\frac{s_i^2}{s_i^2+9(p-d)}\right)V^\top e_1
\end{align}

\noindent
where $s_i$ are the singular values of $A_1$ and operator "diag" produces the diagonal matrix of the set of $s_i$.

\begin{align}
    e_1^\top \PAT e_1 
    &\ge e_1^\top V\text{diag}\left(\frac{(p\gam+1)\tilde{s}_i^2}{(p\gam+1)\tilde{s}_i^2+9(p-d)}\right)V^\top e_1 \\
    &= \sum_{i=1}^d \left( \frac{(p\gam+1)\tilde{s}_i^2}{(p\gam+1)\tilde{s}_i^2+9(p-d)}\right) e_1^\top V_iV_i^\top e_1
\end{align}

\noindent
where $\tilde{s}_i$ are the singular values of the matrix $A_1/(p\gam+1)$. Observe that this matrix has i.i.d. $\N(0,1)$ entries. And thus

\begin{align}
    e_1^\top \PAT  e_1 &\ge \sum_{i=1}^d \left( \frac{(p\gam+1)\tilde{s}_i^2}{(p\gam+1)\tilde{s}_i^2+9(p-d)}\right) V_i(1)^2
\end{align}

\noindent
where $V_i(1)$ is the $1$st element of the $i$th row of $V$. By Theorem 5.39 in \cite{vershynin2010introduction}, $\sqrt{n}-2\sqrt{d} \le \tilde{s}_i \le \sqrt{n}+2\sqrt{d}$ for all $i$ with probability at least $1-2e^{-cd}$ for universal constant $c>0$. This gives

\begin{align}
    e_1^\top \PAT e_1 &\ge \sum_{i=1}^d \Biggl( \frac{(p\gam+1)(\sqrt{n}-2\sqrt{d})^2}{(p\gam+1)(\sqrt{n}+2\sqrt{d})^2+9(p-d)} \Biggl) V_i(1)^2
\end{align}

\noindent
Note that $\sum_{i=1}^d V_i(1)^2 = 1$ by orthogonality of $V$. This gives

\begin{align}
    e_1^\top \PAT e_1 \ge \frac{(p\gam+1)(\sqrt{n}-2\sqrt{d})^2}{(p\gam+1)(\sqrt{n}+2\sqrt{d})^2+9(p-d)}
\end{align}

\noindent
Substituting this into (\ref{one-minus}) gives

\begin{align}
    \|\mathcal{P}_{{A^\top}^\perp}e_1\|^2 &\le 1 - \frac{(p\gam+1)(\sqrt{n}-2\sqrt{d})^2}{(p\gam+1)(\sqrt{n}+2\sqrt{d})^2+9(p-d)} \\
    &\le \frac{(p\gam+1)(\sqrt{n}+2\sqrt{d})^2 - (p\gam+1)(\sqrt{n}-2\sqrt{d})^2 +9(p-d)}{(p\gam+1)(\sqrt{n}+2\sqrt{d})^2} \\
    &\le \frac{8(p\gam+1)\sqrt{nd} + 9(p-d)}{(p\gam+1)(n+4\sqrt{nd}+4d)}
\end{align}

\noindent
Under the assumption that $\gam \ge \frac{1}{\sqrt{nd}}$, this simplifies to

\begin{align}
    \|\mathcal{P}_{{A^\top}^\perp}e_1\|^2 &\le \frac{18\sqrt{nd}}{n+4\sqrt{nd}+4d} \\
    &\le 18\sqrt{\frac{d}{n}}
\end{align}

\end{proof}

\begin{theorem}[Single-Task Risk] \label{single-risk} Suppose $W \in \R^{p \times d}$ follows Assumption \ref{assumption}. Define data matrix $A \in \R^{n \times p}$ and responses $y=A\beta+\eps$ where $\eps \sim \N(0, \sigma^2I_n)$,  $\sigma^2 = \theta^\top(W^\top W + I_d)^{-1}\theta$, and $\beta = W(W^\top W + I_d)^{-1}\theta$ for some $\theta \in \R^d$. Let rows $A_i$ be independent random vectors in $\R^p$ with covariance $\Sigma = WW^\top+I_p$ and $n \ge d, p \ge 20n$, $\gam \ge \frac{1}{\sqrt{nd}}$. Let $\hba$ be the parameters of the minimum norm estimator on $A$ as defined in (\ref{eq-beta-A}). Let $R(f_{\hb})$ be the risk on task A of an estimator with parameters $\hb$. Then there exists constant $c>0$ such that with probability at least $1-10e^{-cd}$, the following holds:

\begin{align}
    R(f_{\hba}) &\le \left(72\sqrt{\frac{d}{n}} + \frac{18n}{p}\right)\|\theta\|^2
\end{align}

\end{theorem}

\begin{proof} By definition,

\begin{align}
    R(f_{\hba}) &= (\beta - \hba)^\top \Sigma (\beta - \hba) + \s^2
\end{align}

\noindent
Substituting the closed-form for $\hba$ gives

\begin{align}
    R(f_{\hba}) &= (\beta - A^\top(AA^\top)^{-1}y)^\top \Sigma (\beta - A^\top(AA^\top)^{-1}y) + \s^2 \\
    &= (\beta - A^\top(AA^\top)^{-1}(A\beta+\eps))^\top \Sigma (\beta - A^\top(AA^\top)^{-1}(A\beta+\eps)) + \s^2 \\
    &= (\beta - A^\top(AA^\top)^{-1}A\beta - A^\top(AA^\top)^{-1} \eps)^\top \Sigma (\beta - A^\top(AA^\top)^{-1}A\beta - A^\top(AA^\top)^{-1} \eps) \nonumber\\&\qquad + \s^2
\end{align}

\noindent
$AA^\top$ has full rank so $A^\top(AA^\top)^{-1}A = \PAT$ where $\PAT$ is the orthogonal projection onto the range of $A^\top$. This gives

\begin{align}
    R(f_{\hba}) &= (\beta - \mathcal{P}_{A^\top}\beta - A^\top(AA^\top)^{-1} \eps)^\top \Sigma (\beta - \mathcal{P}_{A^\top}\beta - A^\top(AA^\top)^{-1} \eps) \nonumber\\&\qquad + \s^2 \\
    &= (\mathcal{P}_{{A^\top}^\perp}\beta - A^\top(AA^\top)^{-1} \eps)^\top \Sigma (\mathcal{P}_{{A^\top}^\perp}\beta - A^\top(AA^\top)^{-1} \eps) + \s^2
\end{align}

\noindent
By Assumption \ref{assumption}, $\Sigma = p\gamma\mathcal{P}_W+I_p$, and hence

\begin{align}
    R(f_{\hba}) &=  (\mathcal{P}_{{A^\top}^\perp}\beta - A^\top(AA^\top)^{-1} \eps)^\top (p\gamma\mathcal{P}_W+I_p) (\mathcal{P}_{{A^\top}^\perp}\beta - A^\top(AA^\top)^{-1} \eps) \nonumber\\&\qquad + \s^2
\end{align}

\noindent
Directly,

\begin{align}
    R(f_{\hba}) &=  (\mathcal{P}_{{A^\top}^\perp}\beta - A^\top(AA^\top)^{-1} \eps)^\top ((p\gam+1)\mathcal{P}_W + \mathcal{P}_{W^\perp}) (\mathcal{P}_{{A^\top}^\perp}\beta - A^\top(AA^\top)^{-1} \eps) \nonumber \\&\qquad + \s^2 \\
    &= (p\gam+1)(\mathcal{P}_{{A^\top}^\perp}\beta - A^\top(AA^\top)^{-1} \eps)^\top \mathcal{P}_W (\mathcal{P}_{{A^\top}^\perp}\beta - A^\top(AA^\top)^{-1} \eps) \nonumber\\&\qquad + (\mathcal{P}_{{A^\top}^\perp}\beta - A^\top(AA^\top)^{-1} \eps)^\top \mathcal{P}_{W^\perp} (\mathcal{P}_{{A^\top}^\perp}\beta - A^\top(AA^\top)^{-1} \eps) \nonumber\\&\qquad + \s^2 \\
    &= (p\gam+1)\|\mathcal{P}_W (\mathcal{P}_{{A^\top}^\perp}\beta - A^\top(AA^\top)^{-1} \eps)\|^2 \nonumber\\&\qquad + \|\mathcal{P}_{W^\perp} (\mathcal{P}_{{A^\top}^\perp}\beta - A^\top(AA^\top)^{-1} \eps)\|^2 + \s^2
\end{align}

\noindent
Using $\|\mathcal{P}_{W^\perp}\| \le 1$,

\begin{align}
    R(f_{\hba}) &\le (p\gam+1)\|\mathcal{P}_W (\mathcal{P}_{{A^\top}^\perp}\beta - A^\top(AA^\top)^{-1} \eps)\|^2 \nonumber\\&\qquad + \|\mathcal{P}_{{A^\top}^\perp}\beta - A^\top(AA^\top)^{-1} \eps\|^2 + \s^2
\end{align}

\noindent
Using triangle inequality,

\begin{align}
    R(f_{\hba}) &\le (p\gam+1)(\|\mathcal{P}_W \mathcal{P}_{{A^\top}^\perp}\beta\| + \|\mathcal{P}_WA^\top(AA^\top)^{-1} \eps\|)^2 \nonumber\\&\qquad + (\|\mathcal{P}_{{A^\top}^\perp}\beta\| + \|A^\top(AA^\top)^{-1} \eps\|)^2 + \s^2
\end{align}

\noindent
Using $\|\mathcal{P}_W\| \le 1$, we get 

\begin{align}
    R(f_{\hba}) &\le (p\gam+2)(\|\mathcal{P}_{{A^\top}^\perp}\beta\| + \|A^\top(AA^\top)^{-1} \eps\|)^2 + \s^2 \\
    &= (p\gam+2)(I+II)^2 + III \label{eq123}
\end{align}

\noindent
where

\begin{align}
    I &= \|\mathcal{P}_{{A^\top}^\perp}\beta\| \\
    II &= \|A^\top(AA^\top)^{-1} \eps\| \\
    III &= \s^2
\end{align}

\noindent
\textbf{We bound each of these terms separately, starting with $I$:} \\
\\
WLOG study $\|\mathcal{P}_{{A^\top}^\perp}e_1\|^2\|\beta\|^2$ where $W=\cvect{\sqrt{p\gam}I_d}{0}$. By Lemma \ref{term3}, the following bound holds with probability at least $1-8e^{-cd}$ for universal constant $c>0$,

\begin{align}
    \|\mathcal{P}_{{A^\top}^\perp}e_1\|^2 &\le 18\sqrt{\frac{d}{n}}
\end{align}

\noindent
Multiplying by $\|\beta\|^2$ and taking the square root yields the expression for $I$. By Lemma \ref{beta-norm}, this gives

\begin{align}
    I &\le \|\theta\|\sqrt{\left(18\sqrt{\frac{d}{n}}\right)\frac{p\gam}{(p\gam+1)^2}} \label{eq1}
\end{align}

\noindent
\textbf{Next we bound $II$:} \\
\\
By Lemma F.9 in \cite{goldfarb2023analysis}, $\|A^\top(AA^\top)^{-1}\eps\|^2 \le \frac{n\|\theta\|^2}{p\gam(\sqrt{p-d}-2\sqrt{n})^2}$ with probability at least $1-2e^{-n}$. So 

\begin{align}
    II \le \|\theta\|\frac{\sqrt{n}}{\sqrt{p\gam}(\sqrt{p-d}-2\sqrt{n})}
\end{align}

\noindent
Using $p \ge 20n$ and $n \ge d$,

\begin{align}
    II \le \|\theta\|\frac{2\sqrt{n}}{p\sqrt{\gamma}} \label{eq2}
\end{align}

\noindent
\textbf{Lastly we bound $III$:} \\
\\
By definition,

\begin{align}
    \sigma^2 &= \theta^\top(W^\top W + I_d)^{-1}\theta
\end{align}

\noindent
Using Assumption \ref{assumption},

\begin{align}
    III &= \theta^\top(p\gam I_d + I_d)^{-1}\theta \\
    &= \frac{1}{p\gam+1}\|\theta\|^2 \label{eq3}
\end{align}

\noindent
\textbf{Substituting (\ref{eq1}), (\ref{eq2}), and (\ref{eq3}) into (\ref{eq123}) gives:}

\begin{align}
    R(f_{\hba}) &\le (p\gam+2)\left(\sqrt{18\sqrt{\frac{d}{n}}\frac{p\gam}{(p\gam+1)^2}}+ \frac{2\sqrt{n}}{p\sqrt{\gamma}} \right)^2\|\theta\|^2 +\frac{1}{p\gam+1}\|\theta\|^2 \\
    &\le (p\gam+2)\left(36\sqrt{\frac{d}{n}}\frac{p\gam}{(p\gam+1)^2}+ \frac{8n}{p^2\gam}\right)\|\theta\|^2 +\frac{1}{p\gam+1}\|\theta\|^2 \\
\end{align}

\noindent
by Jensen's inequality for the quadratic function. It follows that

\begin{align}
    R(f_{\hba}) &\le 36\sqrt{\frac{d}{n}}\frac{p\gam+2}{p\gam}\|\theta\|^2 + \left(\frac{8n(p\gam+2) + p}{p^2\gam}\right)\|\theta\|^2
\end{align}

\noindent
Using $\gam \ge \frac{1}{\sqrt{nd}}$, this simplifies to

\begin{align}
    R(f_{\hba}) &\le \left(72\sqrt{\frac{d}{n}} + \frac{18n}{p}\right)\|\theta\|^2
\end{align}

\end{proof}

\begin{theorem}[Terminal Forgetting] \label{term-forgetting} Suppose $W \in \R^{p \times d}$ follows Assumption \ref{assumption}. Define data matrix $A \in \R^{n \times p}$ and responses $y=A\beta+\eps$ where $\eps \sim \N(0, \sigma^2I_n)$,  $\sigma^2 = \theta^\top(W^\top W + I_d)^{-1}\theta$, and $\beta = W(W^\top W + I_d)^{-1}\theta$ for some $\theta \in \R^d$. Let rows $A_i$ be independent random vectors in $\R^p$ with covariance $\Sigma = WW^\top+I_p$ and $n \ge d, p \ge max(17n, 1/\gam)$. Let $O$ be a random $p \times p$ orthogonal matrix and $B = AO^\top$. Let $\hba$ be the parameters of the minimum norm estimator on $A$ as defined in (\ref{eq-beta-A}), and $\hbba$ be the parameters of the estimator on $B$ using $\hba$ as initialization as defined in (\ref{eq-beta-BA}). Let $R(f_{\hb})$ be the risk on task A of an estimator with parameters $\hb$. Then there exists constant $c>0$ such that with probability at least $1-10e^{-cd}$, the following holds:

\begin{align}
    R(f_{\hbba}) - R(f_{\hba}) \le \left(66\sqrt\frac{n}{p} + \frac{12}{p\gam}\right)\|\theta\|^2
\end{align}

\end{theorem}

\begin{proof}
    Theorem F.11 in \cite{goldfarb2023analysis}.
\end{proof}

\begin{theorem}[Terminal Task A Risk] \label{term-A} Suppose $W \in \R^{p \times d}$ follows Assumption \ref{assumption}. Define data matrix $A \in \R^{n \times p}$ and responses $y=A\beta+\eps$ where $\eps \sim \N(0, \sigma^2I_n)$,  $\sigma^2 = \theta^\top(W^\top W + I_d)^{-1}\theta$, and $\beta = W(W^\top W + I_d)^{-1}\theta$ for some $\theta \in \R^d$. Let rows $A_i$ be independent random vectors in $\R^p$ with covariance $\Sigma = WW^\top+I_p$ and $n \ge d, p \ge 20n$, $\gam \ge \frac{1}{\sqrt{nd}}$. Let $O$ be a random $p \times p$ orthogonal matrix and $B = AO^\top$. Let $\hba$ be the parameters of the minimum norm estimator on $A$ as defined in (\ref{eq-beta-A}), and $\hbba$ be the parameters of the estimator on $B$ using $\hba$ as initialization as defined in (\ref{eq-beta-BA}). Let $R(f_{\hb})$ be the risk on task A of an estimator with parameters $\hb$. Then there exists constant $c>0$ such that with probability at least $1-20e^{-cd}$, the following holds:

\begin{align}
    R(f_{\hbba}) \le \left(72\sqrt{\frac{d}{n}} + 96\sqrt{\frac{n}{p}}\right)\|\theta\|^2
\end{align}

\end{theorem}

\begin{proof} Adding the results from Theorem \ref{single-risk} and Theorem \ref{term-forgetting} gives the following bound with probability at least $1-8e^{-n/2} - 12e^{-cd}$:

\begin{align}
    R(f_{\hbba}) &\le \left(72\sqrt{\frac{d}{n}} + \frac{18n}{p}\right)\|\theta\|^2 + \left(66\sqrt\frac{n}{p} + \frac{12}{p\gam}\right)\|\theta\|^2 \\
    &\le \left(72\sqrt{\frac{d}{n}} + 84\sqrt{\frac{n}{p}} + \frac{12}{p\gam}\right)\|\theta\|^2
\end{align}

\noindent
where the final inequality uses $\frac{n}{p} \le \sqrt{\frac{n}{p}}$. Using $\gam \ge \frac{1}{\sqrt{nd}}$ and $\frac{\sqrt{n}}{p} \le \sqrt{\frac{n}{p}}$, this simplifies to

\begin{align}
    R(f_{\hbba}) &\le \left(72\sqrt{\frac{d}{n}} + 96\sqrt{\frac{n}{p}}\right)\|\theta\|^2
\end{align}
    
\end{proof}

\begin{theorem}[Scale of Forgetting]\label{scale-forgetting} Suppose $W \in \R^{p \times d}$ follows Assumption \ref{assumption}. Define data matrix $A \in \R^{n \times p}$ and responses $y=A\beta+\eps$ where $\eps \sim \N(0, \sigma^2I_n)$,  $\sigma^2 = \theta^\top(W^\top W + I_d)^{-1}\theta$, and $\beta = W(W^\top W + I_d)^{-1}\theta$ for some $\theta \in \R^d$. Let rows $A_i$ be independent random vectors in $\R^p$ with covariance $\Sigma = WW^\top+I_p$ and $n \ge d, p \ge 20n$, $\gam \ge \frac{1}{\sqrt{nd}}$. Let $O$ be a random $p \times p$ orthogonal matrix and $B = AO^\top$. Let $\hba$ be the parameters of the minimum norm estimator on $A$ as defined in (\ref{eq-beta-A}), and $\hbba$ be the parameters of the estimator on $B$ using $\hba$ as initialization as defined in (\ref{eq-beta-BA}). Let $R(f_{\hb})$ be the risk on task A of an estimator with parameters $\hb$. Then there exists constant $c>0$ such that with probability at least $1-20e^{-cd}$, the following holds:

\begin{align}
    \frac{R(f_{\hbba}) - R(f_{\hba})}{R(f_{0}) - R(f_{\hba})} \le \frac{78\sqrt\frac{n}{p}}{1 - 72\sqrt{\frac{d}{n}} - \frac{18 n}{p}}
\end{align}

\end{theorem}

\begin{proof} Using Lemma \ref{null-risk}, Theorem \ref{term-forgetting}, and Theorem \ref{single-risk}, there exists constant $c>0$ such that with probability at least $1-20e^{-cd}$, the following holds:

\begin{align}
    \frac{R(f_{\hbba}) - R(f_{\hba})}{R(f_{0}) - R(f_{\hba})} &\le \frac{\left(66\sqrt\frac{n}{p} + \frac{12}{p\gam}\right)\|\theta\|^2}{\|\theta\|^2 - \left(72\sqrt{\frac{d}{n}} + \frac{18 n}{p}\right)\|\theta\|^2} \\
    &= \frac{66\sqrt\frac{n}{p} + \frac{12}{p\gam}}{1 - 72\sqrt{\frac{d}{n}} - \frac{18 n}{p}} \\
    &\le \frac{78\sqrt\frac{n}{p}}{1 - 72\sqrt{\frac{d}{n}} - \frac{18 n}{p}}
\end{align}

\noindent
where the last inequality uses $\gam \ge \frac{1}{\sqrt{nd}}$ and $\frac{\sqrt{n}}{p} \le \sqrt{\frac{n}{p}}$.

\end{proof}

\vskip 0.2in
\bibliography{sample}

\begin{thebibliography}{22}
\providecommand{\natexlab}[1]{#1}
\providecommand{\url}[1]{\texttt{#1}}
\expandafter\ifx\csname urlstyle\endcsname\relax
  \providecommand{\doi}[1]{doi: #1}\else
  \providecommand{\doi}{doi: \begingroup \urlstyle{rm}\Url}\fi

\bibitem[Adlam and Pennington(2020)]{adlam2020understanding}
Ben Adlam and Jeffrey Pennington.
\newblock Understanding double descent requires a fine-grained bias-variance
  decomposition.
\newblock \emph{Advances in neural information processing systems},
  33:\penalty0 11022--11032, 2020.

\bibitem[Bartlett et~al.(2020)Bartlett, Long, Lugosi, and
  Tsigler]{bartlett2020benign}
Peter~L Bartlett, Philip~M Long, G{\'a}bor Lugosi, and Alexander Tsigler.
\newblock Benign overfitting in linear regression.
\newblock \emph{Proceedings of the National Academy of Sciences}, 117\penalty0
  (48):\penalty0 30063--30070, 2020.

\bibitem[Belkin et~al.(2020)Belkin, Hsu, and Xu]{belkin2020two}
Mikhail Belkin, Daniel Hsu, and Ji~Xu.
\newblock Two models of double descent for weak features.
\newblock \emph{SIAM Journal on Mathematics of Data Science}, 2\penalty0
  (4):\penalty0 1167--1180, 2020.

\bibitem[Blanc et~al.(2020)Blanc, Gupta, Valiant, and
  Valiant]{blanc2020implicit}
Guy Blanc, Neha Gupta, Gregory Valiant, and Paul Valiant.
\newblock Implicit regularization for deep neural networks driven by an
  ornstein-uhlenbeck like process.
\newblock In \emph{Conference on learning theory}, pages 483--513. PMLR, 2020.

\bibitem[Chen and Liu(2022)]{chen2022lifelong}
Zhiyuan Chen and Bing Liu.
\newblock \emph{Lifelong machine learning}.
\newblock Springer Nature, 2022.

\bibitem[Dar et~al.(2021)Dar, Muthukumar, and Baraniuk]{dar2021farewell}
Yehuda Dar, Vidya Muthukumar, and Richard~G Baraniuk.
\newblock A farewell to the bias-variance tradeoff? an overview of the theory
  of overparameterized machine learning.
\newblock \emph{arXiv preprint arXiv:2109.02355}, 2021.

\bibitem[De~Lange et~al.(2021)De~Lange, Aljundi, Masana, Parisot, Jia,
  Leonardis, Slabaugh, and Tuytelaars]{de2021continual}
Matthias De~Lange, Rahaf Aljundi, Marc Masana, Sarah Parisot, Xu~Jia,
  Ale{\v{s}} Leonardis, Gregory Slabaugh, and Tinne Tuytelaars.
\newblock A continual learning survey: Defying forgetting in classification
  tasks.
\newblock \emph{IEEE transactions on pattern analysis and machine
  intelligence}, 44\penalty0 (7):\penalty0 3366--3385, 2021.

\bibitem[Evron et~al.(2023)Evron, Moroshko, Buzaglo, Khriesh, Marjieh, Srebro,
  and Soudry]{evron2023continual}
Itay Evron, Edward Moroshko, Gon Buzaglo, Maroun Khriesh, Badea Marjieh, Nathan
  Srebro, and Daniel Soudry.
\newblock Continual learning in linear classification on separable data.
\newblock In \emph{International Conference on Machine Learning}, pages
  9440--9484. PMLR, 2023.

\bibitem[Farajtabar et~al.(2020)Farajtabar, Azizan, Mott, and
  Li]{farajtabar2020orthogonal}
Mehrdad Farajtabar, Navid Azizan, Alex Mott, and Ang Li.
\newblock Orthogonal gradient descent for continual learning.
\newblock In \emph{International Conference on Artificial Intelligence and
  Statistics}, pages 3762--3773. PMLR, 2020.

\bibitem[French(1999)]{french1999catastrophic}
Robert~M French.
\newblock Catastrophic forgetting in connectionist networks.
\newblock \emph{Trends in cognitive sciences}, 3\penalty0 (4):\penalty0
  128--135, 1999.

\bibitem[Goldfarb and Hand(2023)]{goldfarb2023analysis}
Daniel Goldfarb and Paul Hand.
\newblock m.
\newblock In \emph{International Conference on Artificial Intelligence and
  Statistics}, pages 2975--2993. PMLR, 2023.

\bibitem[Hastie et~al.(2022)Hastie, Montanari, Rosset, and
  Tibshirani]{hastie2022surprises}
Trevor Hastie, Andrea Montanari, Saharon Rosset, and Ryan~J Tibshirani.
\newblock Surprises in high-dimensional ridgeless least squares interpolation.
\newblock \emph{Annals of statistics}, 50\penalty0 (2):\penalty0 949, 2022.

\bibitem[Kirkpatrick et~al.(2017)Kirkpatrick, Pascanu, Rabinowitz, Veness,
  Desjardins, Rusu, Milan, Quan, Ramalho, Grabska-Barwinska,
  et~al.]{kirkpatrick2017overcoming}
James Kirkpatrick, Razvan Pascanu, Neil Rabinowitz, Joel Veness, Guillaume
  Desjardins, Andrei~A Rusu, Kieran Milan, John Quan, Tiago Ramalho, Agnieszka
  Grabska-Barwinska, et~al.
\newblock Overcoming catastrophic forgetting in neural networks.
\newblock \emph{Proceedings of the national academy of sciences}, 114\penalty0
  (13):\penalty0 3521--3526, 2017.

\bibitem[Li and Hoiem(2017)]{li2017learning}
Zhizhong Li and Derek Hoiem.
\newblock Learning without forgetting.
\newblock \emph{IEEE transactions on pattern analysis and machine
  intelligence}, 40\penalty0 (12):\penalty0 2935--2947, 2017.

\bibitem[Lin et~al.(2023)Lin, Ju, Liang, and Shroff]{lin2023theory}
Sen Lin, Peizhong Ju, Yingbin Liang, and Ness Shroff.
\newblock Theory on forgetting and generalization of continual learning.
\newblock In \emph{International Conference on Machine Learning}, pages
  21078--21100. PMLR, 2023.

\bibitem[Mirzadeh et~al.(2022)Mirzadeh, Chaudhry, Yin, Hu, Pascanu, Gorur, and
  Farajtabar]{mirzadeh2022wide}
Seyed~Iman Mirzadeh, Arslan Chaudhry, Dong Yin, Huiyi Hu, Razvan Pascanu, Dilan
  Gorur, and Mehrdad Farajtabar.
\newblock Wide neural networks forget less catastrophically.
\newblock In \emph{International Conference on Machine Learning}, pages
  15699--15717. PMLR, 2022.

\bibitem[Muthukumar et~al.(2020)Muthukumar, Vodrahalli, Subramanian, and
  Sahai]{muthukumar2020harmless}
Vidya Muthukumar, Kailas Vodrahalli, Vignesh Subramanian, and Anant Sahai.
\newblock Harmless interpolation of noisy data in regression.
\newblock \emph{IEEE Journal on Selected Areas in Information Theory},
  1\penalty0 (1):\penalty0 67--83, 2020.

\bibitem[Nakkiran et~al.(2021)Nakkiran, Kaplun, Bansal, Yang, Barak, and
  Sutskever]{nakkiran2021deep}
Preetum Nakkiran, Gal Kaplun, Yamini Bansal, Tristan Yang, Boaz Barak, and Ilya
  Sutskever.
\newblock Deep double descent: Where bigger models and more data hurt.
\newblock \emph{Journal of Statistical Mechanics: Theory and Experiment},
  2021\penalty0 (12):\penalty0 124003, 2021.

\bibitem[Shin et~al.(2017)Shin, Lee, Kim, and Kim]{shin2017continual}
Hanul Shin, Jung~Kwon Lee, Jaehong Kim, and Jiwon Kim.
\newblock Continual learning with deep generative replay.
\newblock \emph{Advances in neural information processing systems}, 30, 2017.

\bibitem[Vershynin(2010)]{vershynin2010introduction}
Roman Vershynin.
\newblock Introduction to the non-asymptotic analysis of random matrices.
\newblock \emph{arXiv preprint arXiv:1011.3027}, 2010.

\bibitem[Zenke et~al.(2017)Zenke, Poole, and Ganguli]{zenke2017continual}
Friedemann Zenke, Ben Poole, and Surya Ganguli.
\newblock Continual learning through synaptic intelligence.
\newblock In \emph{International conference on machine learning}, pages
  3987--3995. PMLR, 2017.

\bibitem[Zhang et~al.(2021)Zhang, Bengio, Hardt, Recht, and
  Vinyals]{zhang2021understanding}
Chiyuan Zhang, Samy Bengio, Moritz Hardt, Benjamin Recht, and Oriol Vinyals.
\newblock Understanding deep learning (still) requires rethinking
  generalization.
\newblock \emph{Communications of the ACM}, 64\penalty0 (3):\penalty0 107--115,
  2021.

\end{thebibliography}

\end{document}